\documentclass[11pt]{article}

\PassOptionsToPackage{numbers,compress}{natbib}
\usepackage{natbib}
\usepackage[margin=1in]{geometry}

\usepackage[utf8]{inputenc}
\usepackage[T1]{fontenc}
\usepackage{algorithm}
\usepackage{algpseudocode}
\usepackage{hyperref}
\usepackage{url}
\usepackage{booktabs}
\usepackage{amsmath,amssymb,amsfonts,amsthm,mathtools}
\usepackage{graphicx}
\usepackage{microtype}
\usepackage{xcolor}
\usepackage{float}
\usepackage{enumitem}
\hypersetup{hidelinks}



\newtheorem{theorem}{Theorem}

\newtheorem{lemma}{Lemma}

\newtheorem{assumption}{Assumption}

\newtheorem{remark}{Remark}

\newcommand{\X}{\mathcal X}

\newcommand{\abs}[1]{\left\lvert #1 \right\rvert}

\newcommand{\argmax}

\title{From Relaxed Indexability to Exact Indexability: A $t$-Step Approach for Partially Observable Restless Bandits}

\author{
Keqin Liu\thanks{Corresponding author. The work was supported by  Leadership Talent Program (Science and Education) of 
SIP (KJQ2024202).}\\
School of Mathematics and Physics\\
Xi'an Jiaotong-Liverpool University\\
Suzhou 215123, China\\
\texttt{Keqin.Liu@xjtlu.edu.cn}
\and
Qizhen Jia\\
School of Mathematics and Physics\\
Xi'an Jiaotong-Liverpool University\\
Suzhou 215123, China\\
\texttt{Qizhen.Jia21@student.xjtlu.edu.cn}
}

\date{}

\begin{document}

\maketitle

\begin{abstract}
Whittle index policies offer a scalable method for restless multi-armed
bandits, but under partial observability even determining the indifference
subsidy at a single belief requires solving an infinite-horizon belief-state
problem with no closed-form value function. \citet{Liu2025} addresses this
difficulty by linearizing the unknown decision boundary, leading to a linear system and a closed-form approximate Whittle index.
However, the resulting threshold uses only a one-step active--passive
comparison and does not account for longer-horizon continuation values.

We extend this framework to a \emph{$t$-step lookahead threshold policy}.
For each subsidy $m$, the threshold is defined by the active-minus-passive
advantage under $t$-step finite-horizon value iteration. At $t=1$, the
threshold is $m$-independent and recovers the linear threshold of
\citet{Liu2025}; for $t>1$, it becomes subsidy-dependent through the induced
first-crossing structure and tracks the exact decision boundary more closely.
The proposed algorithmic framework does not require indexability as an input and includes
an indexability verification. Under the original
Whittle indexability, we prove that the $t$-step approximate Whittle index
converges geometrically to the exact Whittle index,
\[
|\widehat W_t(\omega)-W(\omega)|=O(\beta^t).
\]
Numerically, all 2,715 tested three-state instances are verified with computable priority index functions according to the proposed criterion. The P95 index error decreases from
$2.18\times10^{-2}$ at $t=1$ to $8.93\times10^{-4}$ at $t=8$. In an
exact-comparable instance with $\beta=0.9999$, $t=2$ already recovers the
exact Whittle-index ordering. Moderate-depth threshold policies also outperform the one-step baseline and remain close to the optimal
dynamic-programming benchmark, while runtime grows mildly with $t$.
\end{abstract}

\section{Introduction}

Whittle index policies offer a scalable approach to restless bandits by decoupling the multi-armed problem into single-arm subsidy problems \citep{Whittle1988,WeberWeiss1990,BertsimasNinoMora2000}. For a fully observable Markovian arm in state~$x$, the Whittle index $W(x)$ is the subsidy that makes the controller indifferent between activating the arm and leaving it passive. Ranking arms by these indices yields a policy that, under some large-system conditions, is asymptotically optimal \citep{WeberWeiss1990}.

Computation is simplest in the fully observable setting: the state space is finite, the dynamic programming objects are finite-dimensional, and statewise threshold structure facilitates index evaluation. Under partial observability, this finite-state structure is lost. The single-arm problem can still be formulated as a belief MDP, but the state becomes a belief vector~$\omega$ on a continuous simplex \citep{Sondik1978,Puterman1994}. 
The Whittle indifference condition therefore involves the value function of an infinite-horizon belief-state POMDP rather than a finite-state value. For general POMDPs, standard computational methods resort to finite-horizon and finite-state approximations of the continuous belief MDP \citep{Sondik1978,SaldiYukselLinderPOMDP2018}. 
In the partially observable restless bandit setting, recent Whittle index methods avoid solving the full belief-state POMDP by using boundary approximation \citep{Liu2025}.

\citet{Liu2025} obtains a computable approximate index for partially observable restless bandits with $K>2$ hidden states by replacing the unknown decision boundary with a one-step linearized threshold function $r(\omega)=\omega B^\top$. Combined with first-crossing times under passive belief updates, this yields a finite linear system whose solution gives a closed-form approximate Whittle index. The construction is computationally attractive but uses only a one-step lookahead at the active-passive comparison: continuation values beyond the immediate reward $\omega B^\top$ are not represented in the threshold rule. \citet{Liu2025} notes that the threshold can in principle be sharpened by a $t$-step active-passive comparison, but neither constructs the resulting index nor analyzes its consistency.

This paper develops that extension and analyzes its convergence. We define a
$t$-step lookahead threshold family $\{r_{t,m}\}_{t\ge1}$, where
$r_{t,m}(x)=Q_t^{A,\beta,m}(x)-Q_t^{P,\beta,m}(x)$ is the
active-minus-passive advantage under $t$-step finite-horizon value iteration.
The centered comparison $r_{t,m}(x)>r_{t,m}(\omega)=0$ replaces the one-step
linear threshold. At $t=1$, the threshold is independent of $m$ and recovers
the linearized threshold of \citet{Liu2025}; for $t>1$, the threshold is shaped by~$m$ and tracks the exact decision boundary more closely. The main
algorithm does not require indexability as an input: it computes the
finite-$t$ approximate index directly from the resulting threshold and
first-crossing linear system. When multiple candidate subsidy solutions
arise, a slight modification of the algorithm additionally performs a numerical indexability check by
comparing these solutions and verifying whether their maximum pairwise
difference is within a prescribed tolerance $\varepsilon_{\mathrm{ind}}$.

Our main theoretical result concerns convergence to the exact Whittle index.
Under the original Whittle indexability, we prove that, for all sufficiently
large $t$,
\[
\bigl|\widehat W_t(\omega)-W(\omega)\bigr|=O(\beta^t).
\]
Thus, the indexability assumption is used for the geometric convergence
guarantee rather than for implementing Algorithm~\ref{alg:approx-index-computation}. Numerical
experiments on three-state instances complement this theoretical result.
All 2,715 tested instances are numerically verified with computable index functions according to the proposed algorithmic framework. Moreover, $\widehat W_t$ converges to the high-depth
reference $\widehat W_{15}$, with the P95 error decreasing from
$2.18\times10^{-2}$ at $t=1$ to $8.93\times10^{-4}$ at $t=8$. In an
exact-comparable three-arm instance with $\beta=0.9999$, $t=2$ already
recovers the exact Whittle-index ordering. On a six-arm finite-horizon
instance, the $t=2$ and $t=5$ threshold-index policies outperform the
one-step threshold and the myopic baseline and remain close to the optimal
dynamic-programming solution. Runtime grows mildly with $t$, indicating that
moderate depths provide a practical accuracy--cost tradeoff.

\section{Related Work}

\paragraph{Index policies for restless bandits.}
The classical index-policy literature begins with Gittins indices for rested bandits \citep{Gittins1979} and the Lagrangian relaxation for restless bandits \citep{Whittle1988}. Subsequent structural foundations include the LP relaxation and primal-dual heuristic of \citet{BertsimasNinoMora2000}, the conservation-law and extended-polymatroid view of indexable systems \citep{BertsimasNinoMora1996}, partial-conservation-law methods \citep{NinoMora2001,NinoMora2002}, and marginal-productivity indices \citep{NinoMora2007}. These works characterize indexability through structured policy families, conservation laws, and marginal reward-resource measures.

\paragraph{PCL methods and high-dimensional belief models.}
The partial conservation law and marginal-productivity frameworks establish indexability and compute priority indices through reward and resource measures associated with structured families of single-arm policies \citep{NinoMora2001,NinoMora2002,NinoMora2007}. In the PCL verification framework, threshold policies are defined with respect to an ordered state, and the corresponding discounted reward and resource metrics are used to establish indexability. The partially observable models in \citet{LiuJia2025} preserve a tractable belief-state structure that allows the required threshold and renewal calculations. In our model with $K>2$, however, the belief belongs to the high-dimensional simplex $\Delta_{K-1}$, and there is no known scalar ordering under which the optimal passive sets form a nested threshold family. Applying PCL would therefore first require identifying such an ordering and deriving the corresponding threshold-policy performance metrics. This is precisely the unknown boundary problem in our setting. Building on \citet{Liu2025}, we instead use the $t$-step comparison
\[
r_{t,m}(x)>r_{t,m}(\omega)=0
\]
to construct a computable approximation of the unknown decision boundary $C(m)$ and then compute the approximate Whittle index. This construction provides a theoretical foundation for future applications of PCL methods to high-dimensional partially observable models.

\paragraph{Partially observable restless bandits.}
Partial observability replaces the finite state by a belief state, turning indexability into a statement about sets in a continuous simplex. Related models include dynamic multichannel access \citep{LiuZhao2010}, reset processes \citep{LiuWeberZhao2011}, imperfect observation \citep{LiuWeberZhang2024}, and Kalman-filter restless bandits \citep{DanceSilander2015}. Most of these works either restrict to two hidden states or impose structural assumptions (threshold structure, collapsing observation, restart) that make the exact Whittle index analytically available. Our setting follows \citet{Liu2025} in the partial observability model, where the exact boundary is not available and approximation is necessary.

The closest work is \citet{Liu2025}, who linearizes the unknown
boundary $C(m)$ to $r(\omega)=\omega B^\top=m$, computes first-crossing
times $L(\cdot,\omega)$ under passive dynamics, and obtains a
closed-form approximate Whittle index from a system of linear equations.
The same paper observes that a $t$-step active-passive comparison would
sharpen the boundary approximation but leaves both the construction and
the convergence analysis open. We close this gap by constructing the
subsidy-dependent $t$-step threshold, deriving the corresponding
indexability test and index formula, and establishing its geometric
convergence to the exact Whittle index.

\paragraph{Recent learning and scalability work for RMABs.}
Recent NeurIPS work studies learning and scalable decision making for restless bandits, including Whittle-index learning \citep{XiongLi2023}, average-reward restless bandits without the global-attractor assumption \citep{HongEtAl2023}, and regret analysis for restless-bandit learning \citep{JungTewari2019}. Implementation-oriented work studies belief deduplication, factorization reuse, and vectorization for Whittle-index computation under partial observability \citep{JiaLiu2026}.

\section{Model and indexability}
\label{sec3}

We study a $K$-dimensional state-revealing partially observable restless bandit. Following Whittle's relaxation~\citep{Whittle1988}, we replace the activation constraint by a subsidy $m$ for passivity.
The corresponding relaxed objective is the Lagrangian relaxation of the activation constraint. Because this objective is additive across arms and the arm dynamics are independent, the relaxed problem decomposes into $N$ independent single-arm problems.
We therefore focus on the single arm problem. The hidden state space is $\{0,1,\ldots,K-1\}$, the transition matrix is $P=\{p_{i,j}\}_{i,j=0}^{K-1}$, the discount factor is $\beta\in(0,1)$, and the reward vector is
\[
B=[B_0,B_1,\ldots,B_{K-1}],\qquad
0=B_0\le B_1\le\cdots\le B_{K-1}\le 1.
\]
A belief state is a row vector $\omega\in\X=\Delta_{K-1}$ (a $(K-1)$-dimensional simplex).  If the arm is passive, the belief evolves as $P^k(\omega)=\omega P^k$.  If the arm is active, the hidden state is observed and the next belief resets to the corresponding transition row $p_i=[p_{i,0},\ldots,p_{i,K-1}]$.

These belief dynamics induce a single-arm belief MDP. At belief $\omega$, activating the arm yields the expected immediate reward $\omega B^\top$, whereas leaving it passive yields the subsidy $m$. Let $V_{\beta,m}(\omega)$ be the exact infinite-horizon single-arm value:
\begin{align}
V_{\beta,m}(\omega)&=\max\{V_{\beta,m}(\omega;u=1),V_{\beta,m}(\omega;u=0)\},\\
V_{\beta,m}(\omega;u=1)&=\omega B^\top + \beta\omega
\bigl(V_{\beta,m}(p_0),\ldots,V_{\beta,m}(p_{K-1})\bigr)',\\
V_{\beta,m}(\omega;u=0)&=m+\beta V_{\beta,m}(P^1(\omega)).
\end{align}
To compare the two actions at a belief $\omega$, define the exact
action value gap as
\begin{equation}
\Delta_{\beta,m}(\omega):=V_{\beta,m}(\omega;u=1)-V_{\beta,m}(\omega;u=0).
\label{eq:exact-gap-main}
\end{equation}
Activation is optimal when $\Delta_{\beta,m}(\omega)\ge 0$ and passivity is optimal when $\Delta_{\beta,m}(\omega)\le 0$.

For the belief domain $\X$ and the subsidy interval $M:=[0,1]$, define the exact passive set
\[
P(m):=\{\omega\in \X:\Delta_{\beta,m}(\omega)\le0\}.
\]
Define also the exact active set and decision boundary,
\[
A(m):=\{\omega\in \X:\Delta_{\beta,m}(\omega)>0\},\qquad
C(m):=\{\omega\in \X:\Delta_{\beta,m}(\omega)=0\}.
\]

We say that the single-arm problem is exactly indexable on $\X$ if,
for every $\omega\in \X$, there exists a unique subsidy
$W(\omega)\in M$ such that
\[
\Delta_{\beta,m}(\omega)\ge0
\quad\text{for }m\le W(\omega),
\qquad
\Delta_{\beta,m}(\omega)\le0
\quad\text{for }m\ge W(\omega).
\]

The unique indifference subsidy $W(\omega)$ is the exact Whittle index of the belief state $\omega$. Computing $W(\omega)$ directly requires characterizing the exact decision boundary $C(m)$. Given this boundary, the value functions can be represented through the first time at which a passive belief trajectory enters the active region. For $K>2$, however, $C(m)$ lies in the continuous belief simplex and is generally unavailable in closed form. Consequently, the associated first-crossing times cannot be obtained directly from the exact problem. To circumvent this difficulty, we constructed a $t$-step threshold policy to approximate the exact decision boundary as developed in the next section.

\section{\texorpdfstring{$t$}{t}-step threshold policy 
and approximate Whittle index}
\label{sec4}

Motivated by the difficulty identified in Section~\ref{sec3}, we approximate the exact decision boundary using a $t$-step approximate threshold. We first define the corresponding threshold policy and first-crossing times, then solve the linear equation system to get the approximate Whittle index.

\paragraph{Finite-horizon value functions.}
For a lookahead depth $t\ge1$ and passivity subsidy $m\in M$,
define the finite-horizon value functions
$\{J_h^{\beta,m}\}_{h=0}^{t}$ by
\begin{align}
J_0^{\beta,m}(\omega) &:= 0, \\
Q_h^{A,\beta,m}(\omega)
&:=
\omega B^\top
+ \beta \sum_{i=0}^{K-1} \omega_i J_{h-1}^{\beta,m}(p_i), \\
Q_h^{P,\beta,m}(\omega)
&:=
m + \beta J_{h-1}^{\beta,m}(P^1(\omega)), \\
J_h^{\beta,m}(\omega)
&:=
\max\{Q_h^{A,\beta,m}(\omega),\, Q_h^{P,\beta,m}(\omega)\},
\qquad h = 1, \ldots, t.
\end{align}
The $t$-step active--passive gap is
\[
  r_{t,m}(\omega)
  :=
  Q_t^{A,\beta,m}(\omega) - Q_t^{P,\beta,m}(\omega).
\]

Fixing $\omega$ as a point on the approximate threshold, the corresponding $t$-step threshold policy is defined as follows.

\paragraph{$t$-step threshold policy.}
Fix a target belief $\omega\in\mathcal X$ and find a subsidy $m_t\in M$ such that $r_{t,m_t}(\omega)=0$, define the
$t$-step threshold policy
$\pi_t=\pi_t(\omega,m_t)$ by
\[
\pi_t(x)
=
\begin{cases}
1, & r_{t,m_t}(x)>r_{t,m_t}(\omega)=0,\\
0, & r_{t,m_t}(x)\le r_{t,m_t}(\omega)=0.
\end{cases}
\]
Thus, $\omega$ is a point on the approximate threshold,
and the arm is activated whenever the current advantage exceeds
$0$.

When $t=1$, $J_0\equiv0$ gives
\[
r_{1,m}(x)=xB^\top-m.
\]
Consequently,
\[
r_{1,m}(x)>r_{1,m}(\omega)
\quad\Longleftrightarrow\quad
xB^\top>\omega B^\top,
\]
which is independent of $m$ and recovers the linearized threshold
of \citet{Liu2025}.

\paragraph{First-crossing times.}
Under $\pi_t$, a passive arm starting from belief $x$ evolves as 
$x, P^1(x), P^2(x), \ldots$ until the threshold is first crossed.  
Define the first-crossing time
\begin{equation}
  L_t(x, \omega)
  :=
  \min_{0 \le k < \infty}\bigl\{k \ge 0 : r_{t,m_t}(P^k(x)) > 0\bigr\},
  \label{eq:crossing-main}
\end{equation}
and set $L_t(x, \omega; m_t) = \infty$ if no such $k$ exists. Under policy $\pi_t$ and belief points $p_0,p_1,\cdots p_{K-1}$ we have 
\begin{align}
    \hat{V}_{t, m}(p_k) &:= \frac{1 - \beta^{L(p_k, \omega)}}{1 - \beta}m +\beta^{L(p_k, \omega)}\hat{V}_{t, m}\left(P^{L(p_k, \omega)}(p_k); u = 1\right), \nonumber\\
    &\quad \forall~k\in\{0,\dots,K-1\},\label{eq: linearET1}\\[5pt]
    \hat{V}_{t, m}\left(P^{L(p_k, \omega)}(p_k); u = 1\right) &:= P^{L(p_k, \omega)}(p_k)B'+\beta P^{L(p_k, \omega)}(\omega)(\hat{V}_{t, m}(p_0),\dots,\hat{V}_{t, m}(p_{K-1}))', \nonumber\\
    &\quad \forall~k\in\{0,\dots,K-1\},\label{eq: linearET2}
\end{align}

where the value function $\hat{V}_{t, m}(x)$ denotes the infinite horizon value under~$\pi_t$, starting from a belief state~$x$. Then the subsidy~$m$ may be solved as an {\em approximated Whittle index} under the $t$-step threshold policy.
    \begin{eqnarray}
    &\hat{V}_{t, m}(\omega)& = \omega B' + \beta\omega(\hat{V}_{t, m}(p_0),\cdots,\hat{V}_{t, m}(p_{K-1}))' = m + \beta\hat{V}_{t, m}(P^1(\omega)).\label{eq:equal at threshold}
    \end{eqnarray}

For compactness, we write
$L(\cdot) := L(\cdot,\omega),\ f(\cdot) :=
\frac{1 - \beta^{L(\cdot)}}{1 - \beta},\ g(\omega) := \beta^{L(\omega)}P^{L(\omega)}
(\omega)$, and
$$F(P) := \left(
\begin{matrix}
	f(p_0)\\f(p_1)\\\cdots\\f(p_{K-1})
\end{matrix}
\right),\quad G(P) := \left(
\begin{matrix}
	g(p_0)\\g(p_1)\\\cdots\\g(p_{K-1})
\end{matrix}
\right),\quad H(P)=\biggl(\textbf{I}_K - \beta G(P)\biggr)^{-1}.$$

Under Theorem~3 by \cite{Liu2025}, the approximated Whittle index~$\hat{W_t}(\omega)$ for a belief state~$\omega$ is given by
\begin{equation}\label{eq: Approximated Whittle}
	\hat{W_t}(\omega) = \frac{\omega B' - \beta g(\omega P)\biggl[
		\textbf{I}_K + \beta H(P)G(P)
		\biggr]B' + \beta\omega H(P)G(P)B'}
	{1 + \beta f(\omega P) + \beta\biggl[\beta g(\omega P) - \omega\biggr]
		H(P)F(P)},
\end{equation}

Now we are ready to present the general algorithm to solve approximate Whittle index under the $t$-step threshold function, as detailed in Algorithm~\ref{alg:approx-index-computation}.

\begin{algorithm}
\caption{$t$-step approximate Whittle index}
\label{alg:approx-index-computation}
\begin{algorithmic}[1]
\Require primitives $(P,B,\beta)$, belief $\omega$, subsidy interval
$M=[0,1]$, lookahead depth $t$,
first-crossing-time search limit $\ell_{\max}$, bisection tolerance
$\varepsilon_{\mathrm{bis}}$, and maximum bisection iterations
$N_{\mathrm{bis}}$.

\State $m_{\mathrm{low}} \gets 0$, $m_{\mathrm{high}} \gets 1$.
\State Compute $r_{\mathrm{low}} \gets r_{t,m_{\mathrm{low}}}(\omega)$
      and $r_{\mathrm{high}} \gets r_{t,m_{\mathrm{high}}}(\omega)$.
\For{$i=1,\ldots,N_{\mathrm{bis}}$}
    \State $m_{\mathrm{mid}} \gets (m_{\mathrm{low}} + m_{\mathrm{high}})/2$.
    \State Compute $r_{\mathrm{mid}} \gets r_{t,m_{\mathrm{mid}}}(\omega)$.
    \If{$|r_{\mathrm{mid}}| \le \varepsilon_{\mathrm{bis}}$}
        \State $m_0 \gets m_{\mathrm{mid}}$.
        \State \textbf{break}.
    \EndIf
    \If{$r_{\mathrm{mid}} < 0$}
        \State $m_{\mathrm{high}} \gets m_{\mathrm{mid}}$.
        \State $r_{\mathrm{high}} \gets r_{\mathrm{mid}}$.
    \Else
        \State $m_{\mathrm{low}} \gets m_{\mathrm{mid}}$.
        \State $r_{\mathrm{low}} \gets r_{\mathrm{mid}}$.
    \EndIf
\EndFor

\For{$x\in\{p_0,\ldots,p_{K-1},\,P^1(\omega)\}$}
  \State
  \[
  L(x)\gets
  \min\{0\le k\le\ell_{\max}:
  r_{t,m_0}(P^k(x))>0\},
  \]
  or $\infty$ if none.

  \State
  \[
  f(x)\gets\frac{1-\beta^{L(x)}}{1-\beta},
  \qquad
  g(x)\gets\beta^{L(x)}P^{L(x)}(x),
  \]
  using $f(x)=1/(1-\beta)$ and $g(x)=0$ when $L(x)=\infty$.
\EndFor

\State Compute $F(P)$, $G(P)$ and $H(P)$

\State Compute
\begin{equation*}
	\hat{W_t}(\omega) = \frac{\omega B' - \beta g(\omega P)\biggl[
		\textbf{I}_K + \beta H(P)G(P)
		\biggr]B' + \beta\omega H(P)G(P)B'}
	{1 + \beta f(\omega P) + \beta\biggl[\beta g(\omega P) - \omega\biggr]
		H(P)F(P)},
\end{equation*}

\State \Return $\hat{W_t}(\omega)$.
\end{algorithmic}
\end{algorithm}

Having constructed the finite-$t$ approximate index, we next
study its relation to the exact Whittle index as the lookahead
depth~$t$ increases.

\section{Geometric convergence of the approximate Whittle index}
\label{sec:two-source-consistency}

This section studies the accuracy of the $t$-step approximate index compared with exact Whittle index.

\begin{lemma}[Decision-gap approximation]
\label{lem:lookahead-error-app}
For any fixed $m\in M$, as $t\rightarrow\infty$, we have
\[
\sup_{x\in\mathcal X}|r_{t,m}(x)-\Delta_{\beta,m}(x)|=O(\beta^t).
\]
\end{lemma}
\begin{proof}
See Appendix~\ref{app:two-source-consistency}.
\end{proof}

\begin{assumption}[Exact indexability]
\label{ass:strict-indexability}
We assume the problem is exactly indexable, i.e., it satisfies the original Whittle indexability with a unique Whittle index function~\citep{Whittle1988}.
\end{assumption}

Under this assumption, fix a belief $\omega\in \X$ and write
\[m^\star:=W(\omega).\]
We next use this approximation to bound the performance loss of the induced $t$-step threshold policy relative to the optimal policy for the infinite horizon. Recall that~$V_{\beta,m}$ is the infinite-horizon value, $\hat{V}_{t, m}$ is the infinite-horizon value under~$\pi_t$. For simplicity, Let~$V^\star$ denote~$V_{\beta,m}$ and~$\hat{V}_{t}$ denote~$\hat{V}_{t, m}$.

\begin{lemma}[Value-function gap]
\label{lem:value-gap-app}
Under Assumption~\ref{ass:strict-indexability}, at subsidy~$m^\star$, and as $t\rightarrow\infty$, we have \[ \|V^\star-\hat{V}_{t}\|_\infty=O(\beta^t).\]
\end{lemma}
\begin{proof}
By the performance difference lemma \citep[Lemma~6.1]{KakadeLangford2002}, for any starting belief $x_0$,
\[
\hat{V}_{t}(x_0)-V^\star(x_0)
=\mathbb E_{\tau\sim\pi_t}\Big[\sum_{k=0}^\infty\beta^k A^{\pi^\star}(x_k,u_k)\Big],
\]
where the trajectory $\tau=(x_0,u_0,x_1,u_1,\ldots)$ is generated by $\pi_t$, and $A^{\pi^\star}(x,u)=Q^\star(x,u)-V^\star(x)$ is the advantage under the optimal value function $V^\star$ and $Q^\star(x,u)$ the $Q$-function over the infinite horizon.

Along any trajectory generated by $\pi_t$, the action at $x_k$ is $u_k=\pi_t(x_k)$. There are two cases:
\begin{itemize}
\item If $\pi_t(x_k)=\pi^\star(x_k)$, then $u_k$ is optimal and $A^{\pi^\star}(x_k,u_k)=0$.
\item If $\pi_t(x_k)\ne\pi^\star(x_k)$, we have $|A^{\pi^\star}(x_k,u_k)|=|\Delta_{\beta,m^\star}(x_k)|.$ 

Fix a belief $\omega$ and define
\begin{align*}
v_t(m)
&:=
m+\beta J_{t-1}^{\beta,m}(\omega P)
-\omega B^\top
-\beta\omega
\bigl(
J_{t-1}^{\beta,m}(p_0),\ldots,
J_{t-1}^{\beta,m}(p_{K-1})
\bigr)^\top,\\
v^\star(m)
&:=
m+\beta V_{\beta,m}(\omega P)
-\omega B^\top
-\beta\omega
\bigl(
V_{\beta,m}(p_0),\ldots,
V_{\beta,m}(p_{K-1})
\bigr)^\top.
\end{align*}

Both $v_t$ and $v^\star$ are locally invertible with Lipschitz-continuous inverse functions at~$0$, where the
Lipschitz-constant of $v_t^{-1}$ at~$0$ can be made independent of~$t$ for all sufficiently
large~$t$ due to Lemma~\ref{lem:lookahead-error-app}.\footnote{For a detailed discussion on relating the derivative of a value function to passive times, see
\citet{Liu2025}. For simplicity, we do not consider the degenerate case where the linear system under the optimal policy is singular.}
Since $m^\star=W(\omega)$ is the exact Whittle index and~$m_t$ the finite-horizon indifference subsidy, we have
\[
v^\star(m^\star)=0,
\qquad
m^\star=(v^\star)^{-1}(0),
\qquad
m_t=v_t^{-1}(0).
\]

By Lemma~\ref{lem:lookahead-error-app},
\[
|v_t(m^\star)-v^\star(m^\star)|=O(\beta^t).
\]
Since $v^\star(m^\star)=0$, setting $e_t:=v_t(m^\star)$ gives
\[
e_t=O(\beta^t).
\]
Consequently,
\[
v_t^{-1}(e_t)
=
m^\star
=
(v^\star)^{-1}(0),
\qquad
v_t^{-1}(0)=m_t.
\]
Using the local Lipschitz continuity of $v_t^{-1}$ with Lipschitz-constant $C_1>0$ independent of $t$, we obtain
\[
\begin{aligned}
|m_t-m^\star|
&=
|v_t^{-1}(0)-v_t^{-1}(e_t)|\\
&\le C_1|e_t|\\
&=O(\beta^t).
\end{aligned}
\]

By the definition of $r_{t,m}$ and Lemma~1 in \citep{Liu2025}, there exists a constant
$C_2>0$, independent of $x_k$ and $t$, such that
\[
|r_{t,m_t}(x_k)-r_{t,m^\star}(x_k)|
\le C_2|m_t-m^\star|
=
O(\beta^t),
\]

By Lemma~\ref{lem:lookahead-error-app},
\[
|\Delta_{\beta,m^\star}(x_k)-r_{t,m^\star}(x_k)|
=
O(\beta^t).
\]
Combining this estimate with the preceding bound gives
\[
\Delta_{\beta,m^\star}(x_k)
=
r_{t,m^\star}(x_k)+O(\beta^t)
=
r_{t,m_t}(x_k)+O(\beta^t).
\]
Consequently,
\[
|\Delta_{\beta,m^\star}(x_k)-r_{t,m_t}(x_k)|
=
O(\beta^t).
\]
The case we are considering is $\pi_t(x_k)\ne\pi^\star(x_k)$.
If $\pi_t(x_k)=1$ and $\pi^\star(x_k)=0$, then
\[
r_{t,m_t}(x_k)>0,
\qquad
\Delta_{\beta,m^\star}(x_k)\le0.
\]
If $\pi_t(x_k)=0$ and $\pi^\star(x_k)=1$, then
\[
r_{t,m_t}(x_k)\le0,
\qquad
\Delta_{\beta,m^\star}(x_k)>0.
\]
Thus, the two terms have opposite signs and we have $|\Delta_{\beta,m^\star}(x_k)|=O(\beta^t)$.
\end{itemize}

Therefore $|A^{\pi^\star}(x_k,u_k)|=O(\beta^t)$ and
\[
|\hat{V}_{t}(x_0)-V^\star(x_0)|
\le \mathbb E_{\tau\sim\pi_t}\Big[\sum_{k=0}^\infty\beta^k\cdot O(\beta^t)]
=O(\beta^t).
\]
Taking the supremum over~$x_0$ gives the claim.
\end{proof}

The preceding lemma shows that the $t$-step action value gap has an $O(\beta^t)$ residual when evaluated at the exact
Whittle subsidy. It remains to translate this bound into the distance between the $t$-step approximate index and exact Whittle index. We now state the main convergence result.

\begin{theorem}[Geometric convergence]
\label{thm:geometric-convergence}
Under Assumption~\ref{ass:strict-indexability}, as $t\rightarrow\infty$, we have
\[
  \bigl|\hat{W_t}(\omega)-W(\omega)\bigr|
  =
  O(\beta^t).
\]
\end{theorem}

\begin{proof}
Fix~$\omega$ and let~$m_t$ be the subsidy used to construct~$\pi_t$. Write $m_0=\widehat W_t(\omega)$ and $m^\star=W(\omega)$. According to \eqref{eq:equal at threshold}, the approximate Whittle index is solved as the root of
\begin{equation*}
m_0+\beta\hat V_{t,m_0}(\omega P)-\omega B^\top-\beta\omega
\bigl(
\hat V_{t,m_0}(p_0),\ldots,
\hat V_{t,m_0}(p_{K-1})
\bigr)^\top
=0
\end{equation*}
Meanwhile, the exact Whittle index is the root of
\begin{equation*}
m^\star+\beta V_{\beta,m^\star}(\omega P)-\omega B^\top
-\beta\omega
\bigl(
V_{\beta,m^\star}(p_0),\ldots,
V_{\beta,m^\star}(p_{K-1})
\bigr)^\top
=0
\end{equation*}
By Lemma~\ref{lem:value-gap-app} and adapting its proof, we obtain
\begin{equation*}
|m_0 - m^\star| = O(\beta^t),
\end{equation*}
which completes the proof.
\end{proof}

The preceding result establishes the geometric convergence of the $t$-step approximate Whittle index under Assumption~\ref{ass:strict-indexability}. The implementation of Algorithm~\ref{alg:approx-index-computation} in Section~\ref{sec4}, however, does not require indexability as an input. In all our simulations, Algorithm~\ref{alg:approx-index-computation} successfully found an approximate index, and the resulting policy achieved good performance.

\section{Verification of indexability}

Algorithm~\ref{alg:approx-index-computation} uses bisection to locate m such that $r_{t,m}(\omega)=0$, which substantially accelerates the computation. However, it cannot determine whether additional admissible subsidies exist and therefore cannot verify indexability. To address this limitation, we first establish a theory to verify indexability and then develop an algorithm to numerically verify indexability.

For each belief~$\omega$ and lookahead depth~$t$, define the
admissible subsidy set by
\begin{equation}
\mathcal M_t(\omega)
\coloneqq
\left\{
m\in M:
r_{t,m}(\omega)=0
\right\}.
\label{eq:admissible-subsidy-set}
\end{equation}

\begin{theorem}[Indexability verification]
\label{thm:indexability-verification}
The arm is indexable if and only if, for every belief~$\omega$, all admissible subsidies $m\in\mathcal M_t(\omega)$ obtained at depth $t$ converge to
the same unique subsidy~$m$ as $t\to\infty$.
\end{theorem}

\begin{proof}
Suppose that the arm is not indexable.
Then there exists at least one belief $\omega$ associated
with two distinct exact subsidies
\[
m_1^\star\neq m_2^\star.
\]
Let
\[
d\coloneqq |m_1^\star-m_2^\star|>0
\]
and choose
\[
0<\varepsilon<\frac{d}{3}.
\]

By the proofs of Lemma~\ref{lem:lookahead-error-app} and Lemma~\ref{lem:value-gap-app}, there exists an integer $t_\varepsilon$ such
that, for every $t\geq t_\varepsilon$, the admissible solution set
$\mathcal M_t(\omega)$ contains two solutions~$m_{1,t}$ and~$m_{2,t}$
satisfying
\[
|m_{1,t}-m_1^\star|<\varepsilon,
\qquad
|m_{2,t}-m_2^\star|<\varepsilon.
\]

By the reverse triangle inequality,
\begin{align}
|m_{1,t}-m_{2,t}|
&\geq
|m_1^\star-m_2^\star|
-
|m_{1,t}-m_1^\star|
-
|m_{2,t}-m_2^\star| \\
&>
d-2\varepsilon
>
\frac{d}{3}.
\end{align}
Consequently,
\[
\operatorname{diam}\bigl(\mathcal M_t(\omega)\bigr)
\geq
|m_{1,t}-m_{2,t}|
>
\frac{d}{3},
\qquad
t\geq t_\varepsilon.
\]
Hence, there exist two subsequences~$m_{1,t_k}$ and~$m_{2,t_k}$ in $\mathcal{M}_t(\omega)$ converging to two distinct limits~$m_1^*$ and~$m_2^*$. The other direction is similar.
\end{proof}

\begin{algorithm}[h]
\caption{Indexability verification}
\label{alg:indexverification}
\begin{algorithmic}[1]
\Require primitives $(P,B,\beta)$, belief $\omega$, subsidy interval
$M$, maximum lookahead depth $t_{\max}$, subsidy-grid resolution
$\Delta_m$, tolerance $\varepsilon_{\mathrm{ind}}$ and $\varepsilon_{\mathrm{m}}$.

\State Discretize $M$ with resolution $\Delta_m$ to get grid $M_{\Delta_m}$.

\State $\texttt{status}\gets\texttt{not indexable}$.

\For{$t=1,\ldots,t_{\max}$}
    \For{$m\in M_{\Delta_m}$}
        \State Compute
        \[
        r_{t,m}(\omega)
        =
        Q_t^{A,\beta,m}(\omega)
        -
        Q_t^{P,\beta,m}(\omega).
        \]
    \EndFor
    \State Collect all m such that $r_{t,m}(\omega)<\varepsilon_{\mathrm{m}}$ into $\mathcal M_t(\omega)$.
\EndFor

\State Find the smallest $t_n\in\{1,\ldots,t_{\max}-1\}$ such that:
\[
\quad \forall t\in \{t_n,\ldots,t_{\max}\},
\]
and for each such $t$,
\[
|\mathcal M_t(\omega)|=1
\quad\text{or}\quad
\max_{\substack{m,m'\in\mathcal M_t(\omega)\\ m\neq m'}} |m-m'| \le \varepsilon_{\mathrm{ind}}.
\]

\If{such a $t_n$ exists}
    \State $\texttt{status}\gets\texttt{indexable}$.
\EndIf

\State \Return $\texttt{status}$.
\end{algorithmic}
\end{algorithm}

\begin{remark}
Algorithm~\ref{alg:indexverification} discretizes the subsidy interval and searches for all admissible subsidies. It can therefore numerically verify indexability by checking whether, for every belief $\omega$, all subsidies in $\mathcal M_t(\omega)$ converge to the same unique subsidy as $t$ increases.
\end{remark}

\section{Numerical experiments}

The experiments evaluate the accuracy, indexability, policy performance,
and computational cost of the finite-lookahead approximate Whittle index
in partially observable restless bandits. All $2,715$ tested instances are
numerically verified with computable index functions according to the proposed algorithmic framework.

\subsection{Index convergence in $t$}

\begin{table}[H]
\caption{Errors between \(t\)-step approximate Whittle indices and the
high-depth reference \(W^{\rm ref}\) on \(K=3\) samples. Errors are
\(\abs{\widehat W_t-W^{\rm ref}}\).}
\label{tab:depth-concentration}
\centering
\small
\begin{tabular}{rrrrr}
\toprule
\(t\) & median & P90 & P95 & max \\
\midrule
1  & \(3.44{\times}10^{-3}\) & \(1.37{\times}10^{-2}\) & \(2.18{\times}10^{-2}\) & \(6.27{\times}10^{-2}\) \\
2  & \(1.72{\times}10^{-3}\) & \(7.98{\times}10^{-3}\) & \(1.12{\times}10^{-2}\) & \(2.13{\times}10^{-2}\) \\
3  & \(6.16{\times}10^{-4}\) & \(5.87{\times}10^{-3}\) & \(9.13{\times}10^{-3}\) & \(1.18{\times}10^{-2}\) \\
5  & \(7.99{\times}10^{-15}\) & \(2.89{\times}10^{-3}\) & \(4.15{\times}10^{-3}\) & \(1.17{\times}10^{-2}\) \\
8  & \(0\) & \(5.01{\times}10^{-4}\) & \(8.93{\times}10^{-4}\) & \(2.79{\times}10^{-3}\) \\
\bottomrule
\end{tabular}
\end{table}

Table~\ref{tab:depth-concentration} shows that increasing the
lookahead depth substantially reduces the error relative to a
high-depth reference. Because the exact Whittle indices are not
available for these randomly generated instances, we use
\[
W^{\rm ref}(\omega):=\widehat W_{15}(\omega).
\]
The P95 error decreases from \(2.18\times10^{-2}\) at \(t=1\) to
\(8.93\times10^{-4}\) at \(t=8\), while the maximum error decreases
from \(6.27\times10^{-2}\) to \(2.79\times10^{-3}\).

\paragraph{Ranking stability at a high discount factor.}
The lookahead-error bound becomes increasingly conservative as
\(\beta\uparrow1\), suggesting that a larger \(t\) may be required
to obtain a highly precise index value. In an index policy, however,
the indices are used to rank the arms, so \(t\) only needs to be
large enough to stabilize the relevant ranking. We illustrate this
point using a three-arm instance with
\[
\beta=0.9999,\qquad
B=\left(0,\frac{1}{2},1\right),
\]
and initial beliefs
\[
\omega_1=\omega_2=(0.1,0.3,0.6),
\qquad
\omega_3=(0.2,0.6,0.2).
\]
All transition probabilities are strictly positive, and each arm
satisfies
\[
P_i^2=\mathbf{1}\pi_i,
\]
where \(\mathbf{1}=(1,1,1)^\top\) is the all-ones column vector and
\(\pi_i\) is the stationary row distribution of arm \(i\). Hence, for
every initial belief \(\omega\),
\[
\omega P_i^2=\pi_i.
\]
Therefore, all passive belief trajectories merge after two
transitions, and the corresponding belief MDPs have finite closures
of only five or six states. We compute the exact Whittle indices using
the finite-state method of \citet{GastGaujalKhun2023}:
\[
(W_1,W_2,W_3)
=
\left(
\frac{667771}{875548},
\frac{823331}{1085924},
\frac{109999}{215554}
\right)
=
(0.7626891958,\;0.7581847348,\;0.5103083218).
\]
The exact initial ordering is therefore \(1>2>3\).

\begin{table}[H]
\caption{Exact and \(t\)-step approximate Whittle indices for the
three-arm high-discount instance with \(\beta=0.9999\).}
\label{tab:high-beta-ranking}
\centering
\small
\begin{tabular}{lcccc}
\toprule
Method & Arm 1 & Arm 2 & Arm 3 & Ordering \\
\midrule
Exact
& \(0.7626891958\)
& \(0.7581847348\)
& \(0.5103083218\)
& \(1>2>3\) \\
\(t=1\)
& \(0.7195963682\)
& \(0.7500332014\)
& \(0.5000105161\)
& \(2>1>3\) \\
\(t=2\)
& \(0.7536549707\)
& \(0.7503561253\)
& \(0.5030864197\)
& \(1>2>3\) \\
\(t=3\)
& \(0.7626891958\)
& \(0.7565743293\)
& \(0.5103083218\)
& \(1>2>3\) \\
\(t=4\)
& \(0.7626891958\)
& \(0.7581745846\)
& \(0.5103083218\)
& \(1>2>3\) \\
\bottomrule
\end{tabular}
\end{table}

As shown in Table~\ref{tab:high-beta-ranking}, the one-step
approximation reverses the ordering of Arms 1 and 2. Nevertheless,
\(t=2\) already recovers the exact ordering \(1>2>3\) and therefore
selects the same arm as the exact Whittle index policy. The ordering
remains unchanged at \(t=3\) and \(t=4\). Thus, even when the discount
factor is extremely close to one, a small lookahead depth can be
sufficient to stabilize the ranking of arms.

\subsection{Exact-comparable on small-scale experiments}

The previous experiment measures index-level convergence.  We also check whether the high depth of lookahead approximate Whittle index has better performance in a small finite horizon problem where exact dynamic programming is still feasible.  The instance has six arms, budget \(1\), discount factor \(\beta=0.99\), and horizon
up to \(T=6\).  We compare optimal policy, the one-step threshold index, the
\(t=2\) and \(t=5\) threshold-index policies, and myopic allocation.

\begin{figure}[H]
\centering
\includegraphics[width=0.78\linewidth]{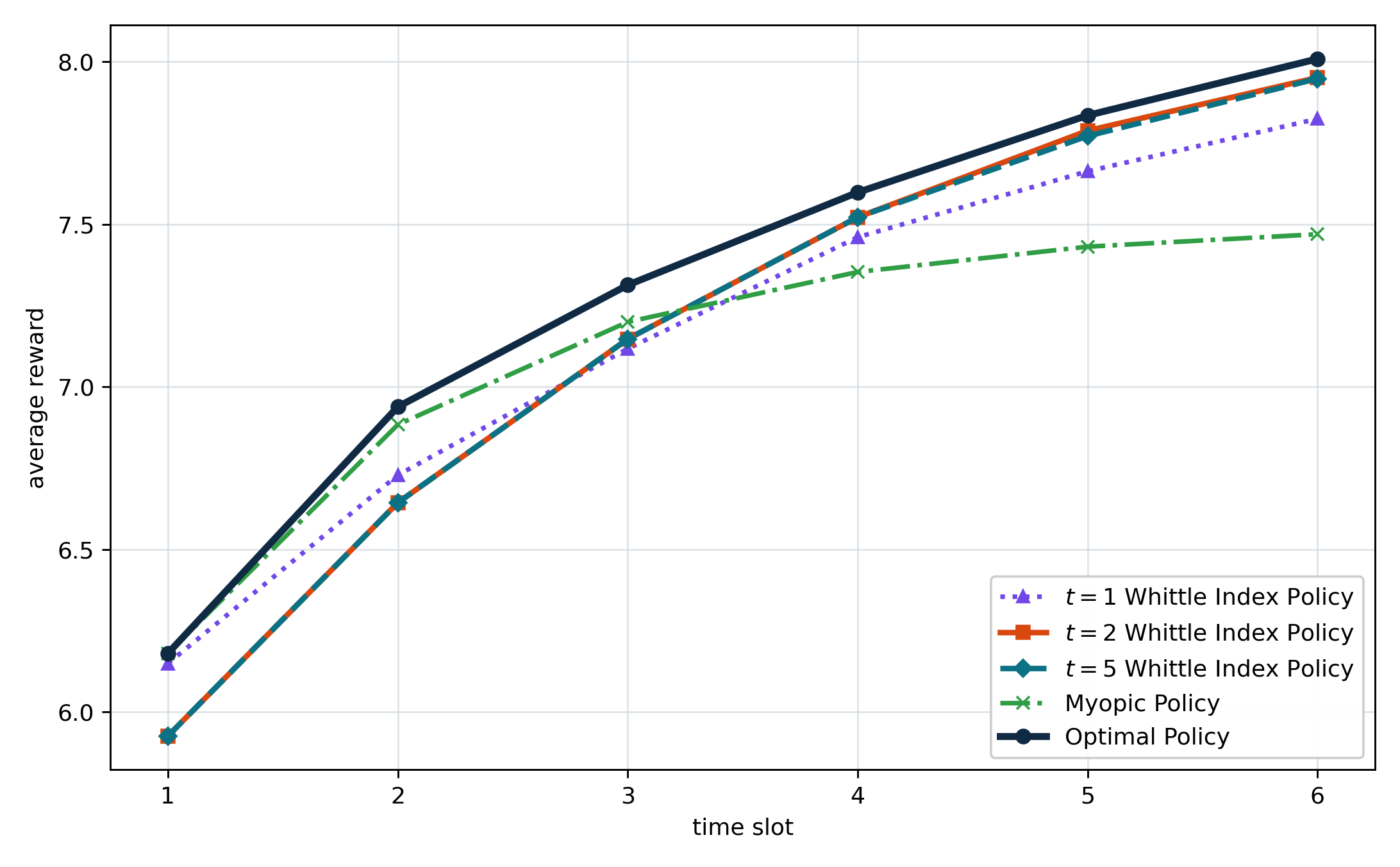}
\caption{small-scale experiments}
\label{fig:policy-sanity-main}
\end{figure}

Figure~\ref{fig:policy-sanity-main} reports that the optimal policy provides the highest
average reward throughout the horizon.  The myopic policy is competitive in the early slots, but its reward curve flattens in the later slots, leaving a significant gap from the optimal benchmark. In contrast, the finite lookahead policies close most of this gap.  The improvement from $t=1$ to $t=2$ is especially clear in the later slots, while the $t=2$ and $t=5$ curves are nearly the same on this instance but still perform better than $t=1$ Whittle index policy. 

\subsection{Runtime of computing different lookahead approximate Whittle index}

\begin{center}
\small
\setlength{\tabcolsep}{4pt}
\begin{tabular}{rrrr}
\toprule
$t$ & samples & median sec/sample & P90 sec/sample \\
\midrule
1 & 80 & 0.0509 & 0.0727 \\
2 & 80 & 0.0834 & 0.0879 \\
3 & 80 & 0.0911 & 0.0957 \\
5 & 80 & 0.1090 & 0.1140 \\
8 & 80 & 0.1372 & 0.1431 \\
15 & 80 & 0.2009 & 0.2052 \\
\bottomrule
\end{tabular}
\end{center}

The per-sample runtime increases predictably with the lookahead depth.  The
median cost rises from \(0.0509\) seconds at \(t=1\) to \(0.2009\) seconds at
\(t=15\), while the P90 values remain close to the medians.  This suggests that
the computation is stable across samples and that deeper lookahead does not
create heavy-tailed runtime behavior in this experiment.  Combined with the depth results, the table indicates that moderate depths such as
\(t=5\) or \(t=8\) offer a reasonable accuracy-cost tradeoff.

\section{Conclusion} We extended the one-step linearized threshold of \citet{Liu2025} to a $t$-step lookahead family and derived the corresponding finite-$t$ approximate Whittle index through the induced first-crossing linear system. The computation itself does not require indexability as an input. When multiple candidate subsidy solutions arise, Algorithm~\ref{alg:indexverification} further provides an indexability verification. Under the original Whittle indexability, we prove that the finite-$t$ approximate index converges geometrically to the exact Whittle index, \[ |\widehat W_t(\omega)-W(\omega)|=O(\beta^t). \] Numerically, all 2,715 tested three-state instances are verified with computable index functions according to the proposed algorithmic framework. Increasing the lookahead depth substantially improves the index accuracy, while a small lookahead depth can already recover the correct index ranking even when $\beta$ is close to one. The resulting threshold-index policies improve upon the one-step and myopic baselines and remain close to the optimal dynamic-programming benchmark on small exact-comparable instances, with a moderate increase in computational cost. The $t$-step threshold and first-crossing construction also provides a computable approximation of the unknown decision boundary and lays a theoretical foundation for future applications of PCL methods to high-dimensional partially observable models.


\bibliographystyle{plainnat}
\bibliography{references}

\clearpage
\appendix

\section{Notation}

This appendix gives the full proofs for the results stated in the main paper.  The belief space is $\X=\Delta_{K-1}$, the passive belief update is $P^1(\omega)=\omega P$, and $p_i$ denotes row $i$ of the transition matrix. 
For subsidy $m$, $\Delta_{\beta,m}(\omega)=V_{\beta,m}(\omega;u=1)-V_{\beta,m}(\omega;u=0)$ is the exact active-minus-passive gap.  The passive set, active set, and exact indifference boundary are
\[
P(m)=\{\omega:\Delta_{\beta,m}(\omega)\le0\},\quad
A(m)=\{\omega:\Delta_{\beta,m}(\omega)>0\},\quad
C(m)=\{\omega:\Delta_{\beta,m}(\omega)=0\}.
\]

\begin{center}
\textbf{Notation used in the main paper and appendix.}

\small
\setlength{\tabcolsep}{4pt}
\begin{tabular}{p{0.28\linewidth}p{0.62\linewidth}}
\toprule
Symbol & Meaning \\
\midrule
$\mathcal X=\Delta_{K-1}$
& Belief simplex over $K$ hidden states \\

$P$, $p_i$
& Transition matrix and its $i$th row \\

$B$
& Reward vector for the active action \\

$\beta$
& Discount factor \\

$m$
& Subsidy for the passive action \\

$V_{\beta,m}$
& Optimal infinite-horizon value function at subsidy $m$ \\

$\Delta_{\beta,m}(x)$
& Infinite-horizon optimal value gap \\

$P(m)$, $A(m)$, $C(m)$
& Passive set, active set, and decision boundary at subsidy $m$ \\

$W(\omega)$, $m^\star$
& Exact Whittle index at $\omega$, with $m^\star=W(\omega)$ \\

$J_h^{\beta,m}$
& Optimal $h$-step finite-horizon value function at subsidy $m$ \\

$Q_h^{A,\beta,m}$, $Q_h^{P,\beta,m}$
& Optimal $h$-step active and passive action values \\

$r_{t,m}(x)$
& $t$-step value gap at belief $x$ \\

$m_t$
& Finite-horizon indifference subsidy satisfying
$r_{t,m_t}(\omega)=0$, used to construct $\pi_t$ \\

$\pi_t$
& Threshold policy defined by
$\pi_t(x)=\mathbf 1\{r_{t,m_t}(x)>0\}$ \\

$L_t(x,\omega)$
& First-crossing time from belief $x$ under consecutive
passive actions, with $\pi_t$ fixed \\

$\hat V_{t,m}$
& Infinite-horizon value function of the fixed policy
$\pi_t$, evaluated at subsidy $m$ \\

$\widehat W_t(\omega)$, $m_0$
& Approximate Whittle index obtained from the
infinite-horizon indifference equation under the fixed policy
$\pi_t$, with $m_0=\widehat W_t(\omega)$ \\
\bottomrule
\end{tabular}
\end{center}

\section{Proof of Lemma 1}
\label{app:two-source-consistency}

\paragraph{Lemma 1 (Lookahead error).}
For any fixed $m\in M$,
\[
\sup_{x\in\mathcal X}|r_{t,m}(x)-\Delta_{\beta,m}(x)|\le \frac{2\beta^t}{1-\beta}=O(\beta^t).
\]

\begin{proof}

\textbf{Step 1: $\|J_t^{\beta,m}-V_{\beta,m}\|_\infty=O(\beta^t)$.}

Define the single-arm Bellman operator $\mathcal B_{\beta,m}$ acting on a function $J:\mathcal X\to\mathbb R$ by
\[
(\mathcal B_{\beta,m}J)(x)
=\max\Big\{\,\underbrace{xB^\top+\beta\sum_{i=0}^{K-1}x_i J(p_i)}_{\text{active value}},\ \underbrace{m+\beta J(P^1(x))}_{\text{passive value}}\,\Big\}.
\]
By construction, $J_t^{\beta,m}=\mathcal B_{\beta,m}J_{t-1}^{\beta,m}$ with $J_0^{\beta,m}\equiv 0$, and $V_{\beta,m}$ is the unique fixed point of $\mathcal B_{\beta,m}$.

The operator $\mathcal B_{\beta,m}$ is a $\beta$-contraction on $\|\cdot\|_\infty$: for any $J,J'$\citet[Chapter~6]{Puterman1994},
\[
\|\mathcal B_{\beta,m}J-\mathcal B_{\beta,m}J'\|_\infty\le \beta\|J-J'\|_\infty,
\]
because the only $J$-dependent terms are the discounted continuations and the operator $\max\{\cdot,\cdot\}$ is non-expansive. Iterating $t$ times,
\[
\|J_t^{\beta,m}-V_{\beta,m}\|_\infty
=\|\mathcal B_{\beta,m}^t J_0-\mathcal B_{\beta,m}^t V_{\beta,m}\|_\infty
\le \beta^t\|J_0-V_{\beta,m}\|_\infty.
\]
Since $B_i\in[0,1]$ and $m\in M=[0,1]$, the one period reward of any policy lies in $[0,1]$, so $\|V_{\beta,m}\|_\infty\le \frac{1}{1-\beta}$. With $J_0\equiv 0$ this gives
\[
\|J_t^{\beta,m}-V_{\beta,m}\|_\infty\le \frac{\beta^t}{1-\beta}=O(\beta^t).
\]

\textbf{Step 2: Transfer the value-function error to action values.}

We now compare the finite-lookahead action values with the exact
infinite-horizon action values. First consider the active branch. By definition,
\[
Q_t^{A,\beta,m}(x)
=
xB^\top+\beta\sum_i x_i J_{t-1}^{\beta,m}(p_i),
\]
whereas
\[
V_{\beta,m}(x;u=1)
=
xB^\top+\beta\sum_i x_i V_{\beta,m}(p_i).
\]
Subtracting the two expressions, the immediate reward \(xB^\top\) cancels:
\[
Q_t^{A,\beta,m}(x)-V_{\beta,m}(x;u=1)
=
\beta\sum_i x_i
\left[
J_{t-1}^{\beta,m}(p_i)-V_{\beta,m}(p_i)
\right].
\]
Taking absolute values and using the triangle inequality,
\[
\begin{aligned}
\left|
Q_t^{A,\beta,m}(x)-V_{\beta,m}(x;u=1)
\right|
&=
\beta
\left|
\sum_i x_i
\left[
J_{t-1}^{\beta,m}(p_i)-V_{\beta,m}(p_i)
\right]
\right| \\
&\le
\beta
\sum_i x_i
\left|
J_{t-1}^{\beta,m}(p_i)-V_{\beta,m}(p_i)
\right|.
\end{aligned}
\]
Since \(x\) is a belief vector, \(x_i\ge 0\) and \(\sum_i x_i=1\). Moreover,
by the definition of the sup norm,
\[
\left|
J_{t-1}^{\beta,m}(p_i)-V_{\beta,m}(p_i)
\right|
\le
\|J_{t-1}^{\beta,m}-V_{\beta,m}\|_\infty .
\]
Therefore
\[
\begin{aligned}
\left|
Q_t^{A,\beta,m}(x)-V_{\beta,m}(x;u=1)
\right|
&\le
\beta
\sum_i x_i
\|J_{t-1}^{\beta,m}-V_{\beta,m}\|_\infty \\
&=
\beta
\|J_{t-1}^{\beta,m}-V_{\beta,m}\|_\infty .
\end{aligned}
\]
Applying Step~1 with \(t-1\) gives
\[
\|J_{t-1}^{\beta,m}-V_{\beta,m}\|_\infty
\le
\frac{\beta^{t-1}}{1-\beta}.
\]
Hence
\[
\left|
Q_t^{A,\beta,m}(x)-V_{\beta,m}(x;u=1)
\right|
\le
\frac{\beta^t}{1-\beta}.
\]
Taking the supremum over \(x\),
\begin{equation}
\sup_{x\in\mathcal X}
\left|
Q_t^{A,\beta,m}(x)-V_{\beta,m}(x;u=1)
\right|
\le
\frac{\beta^t}{1-\beta}.
\label{eq:active-action-error}
\end{equation}

The passive branch is analogous. By definition,
\[
Q_t^{P,\beta,m}(x)
=
m+\beta J_{t-1}^{\beta,m}(P^1(x)),
\]
whereas
\[
V_{\beta,m}(x;u=0)
=
m+\beta V_{\beta,m}(P^1(x)).
\]
Subtracting cancels the subsidy \(m\):
\[
Q_t^{P,\beta,m}(x)-V_{\beta,m}(x;u=0)
=
\beta
\left[
J_{t-1}^{\beta,m}(P^1(x))
-
V_{\beta,m}(P^1(x))
\right].
\]
Therefore
\[
\begin{aligned}
\left|
Q_t^{P,\beta,m}(x)-V_{\beta,m}(x;u=0)
\right|
&=
\beta
\left|
J_{t-1}^{\beta,m}(P^1(x))
-
V_{\beta,m}(P^1(x))
\right| \\
&\le
\beta
\|J_{t-1}^{\beta,m}-V_{\beta,m}\|_\infty \\
&\le
\frac{\beta^t}{1-\beta}.
\end{aligned}
\]
Taking the supremum over \(x\),
\begin{equation}
\sup_{x\in\mathcal X}
\left|
Q_t^{P,\beta,m}(x)-V_{\beta,m}(x;u=0)
\right|
\le
\frac{\beta^t}{1-\beta}.
\label{eq:passive-action-error}
\end{equation}

\textbf{Step 3: Combine the active and passive action-value errors.}

By definition,
\[
r_{t,m}(x)
=
Q_t^{A,\beta,m}(x)-Q_t^{P,\beta,m}(x),
\]
and
\[
\Delta_{\beta,m}(x)
=
V_{\beta,m}(x;u=1)-V_{\beta,m}(x;u=0).
\]
Therefore
\[
\begin{aligned}
r_{t,m}(x)-\Delta_{\beta,m}(x)
&=
\left[
Q_t^{A,\beta,m}(x)-Q_t^{P,\beta,m}(x)
\right]
-
\left[
V_{\beta,m}(x;u=1)-V_{\beta,m}(x;u=0)
\right] \\
&=
\left[
Q_t^{A,\beta,m}(x)-V_{\beta,m}(x;u=1)
\right]
-
\left[
Q_t^{P,\beta,m}(x)-V_{\beta,m}(x;u=0)
\right].
\end{aligned}
\]
Define
\[
E_A(x):=
Q_t^{A,\beta,m}(x)-V_{\beta,m}(x;u=1),
\qquad
E_P(x):=
Q_t^{P,\beta,m}(x)-V_{\beta,m}(x;u=0).
\]
Then
\[
r_{t,m}(x)-\Delta_{\beta,m}(x)=E_A(x)-E_P(x).
\]
By the triangle inequality,
\[
\left|
r_{t,m}(x)-\Delta_{\beta,m}(x)
\right|
=
|E_A(x)-E_P(x)|
\le
|E_A(x)|+|E_P(x)|.
\]
Using \eqref{eq:active-action-error} and \eqref{eq:passive-action-error},
\[
|E_A(x)|
\le
\frac{\beta^t}{1-\beta},
\qquad
|E_P(x)|
\le
\frac{\beta^t}{1-\beta}.
\]
Hence, for every \(x\in\mathcal X\),
\[
\left|
r_{t,m}(x)-\Delta_{\beta,m}(x)
\right|
\le
\frac{\beta^t}{1-\beta}
+
\frac{\beta^t}{1-\beta}
=
\frac{2\beta^t}{1-\beta}.
\]
Taking the supremum over \(x\), we obtain
\[
\sup_{x\in\mathcal X}
\left|
r_{t,m}(x)-\Delta_{\beta,m}(x)
\right|
\le
\frac{2\beta^t}{1-\beta}
=
O(\beta^t).
\]
\end{proof}

\section{Additional experiments}
\label{sec:additional-numerical-audits}

All reported experiments were run on a MacBook M4 Pro with 24GB RAM using CPU only; no GPU or external cluster was used.

\subsection{Additional small-scale experiments}

\begin{figure}[t]
\centering
\includegraphics[width=0.78\linewidth]{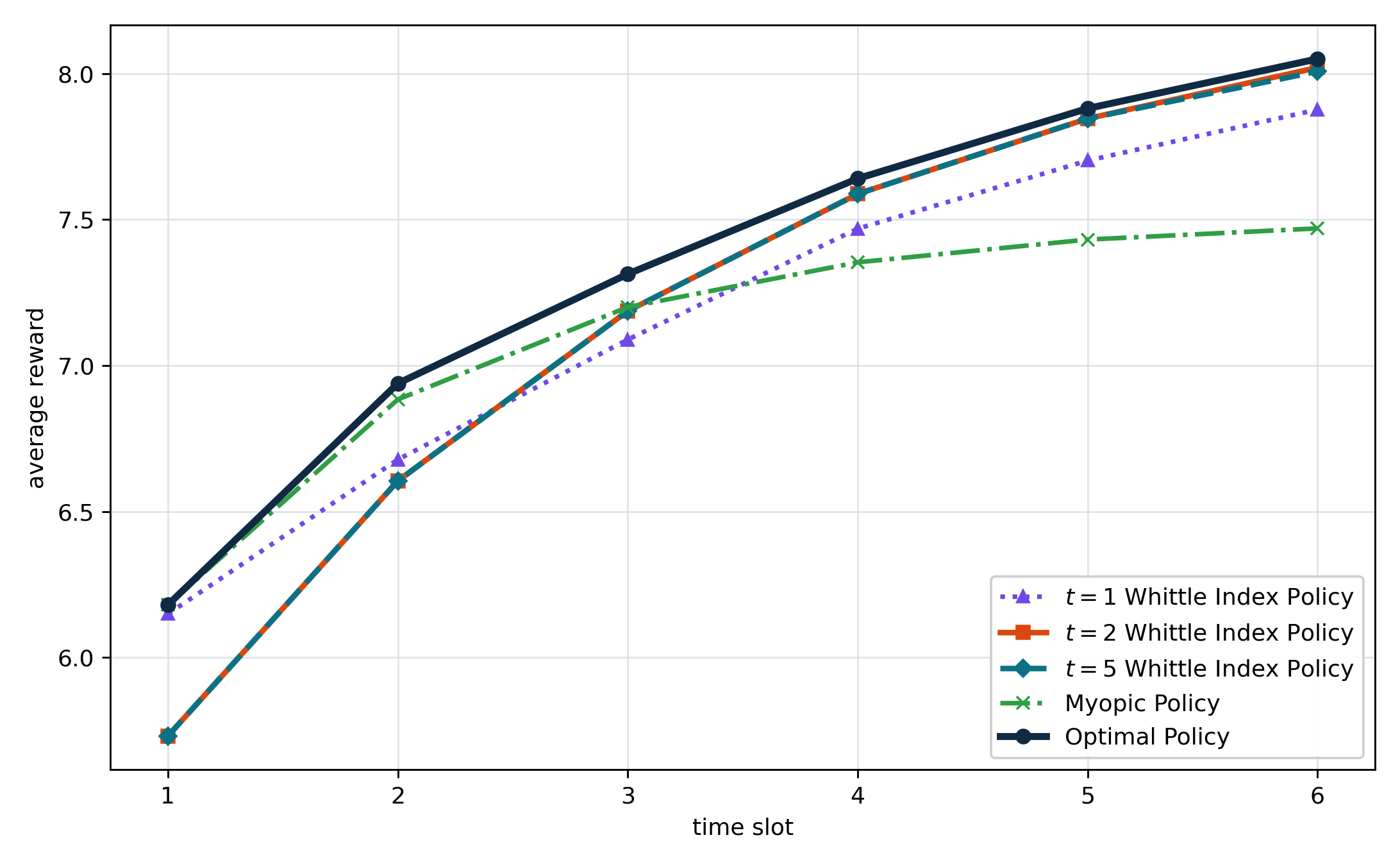}
\caption{Experiment with larger $t=2$ and $t=5$ versus $t=1$ separation.  This instance highlights that the t-step threshold can repair a one-step threshold loss, while remaining close to optimal policy in the finite-horizon comparison.}
\label{fig:appendix-policy-t1-separati}
\end{figure}

\begin{figure}[t]
\centering
\includegraphics[width=0.78\linewidth]{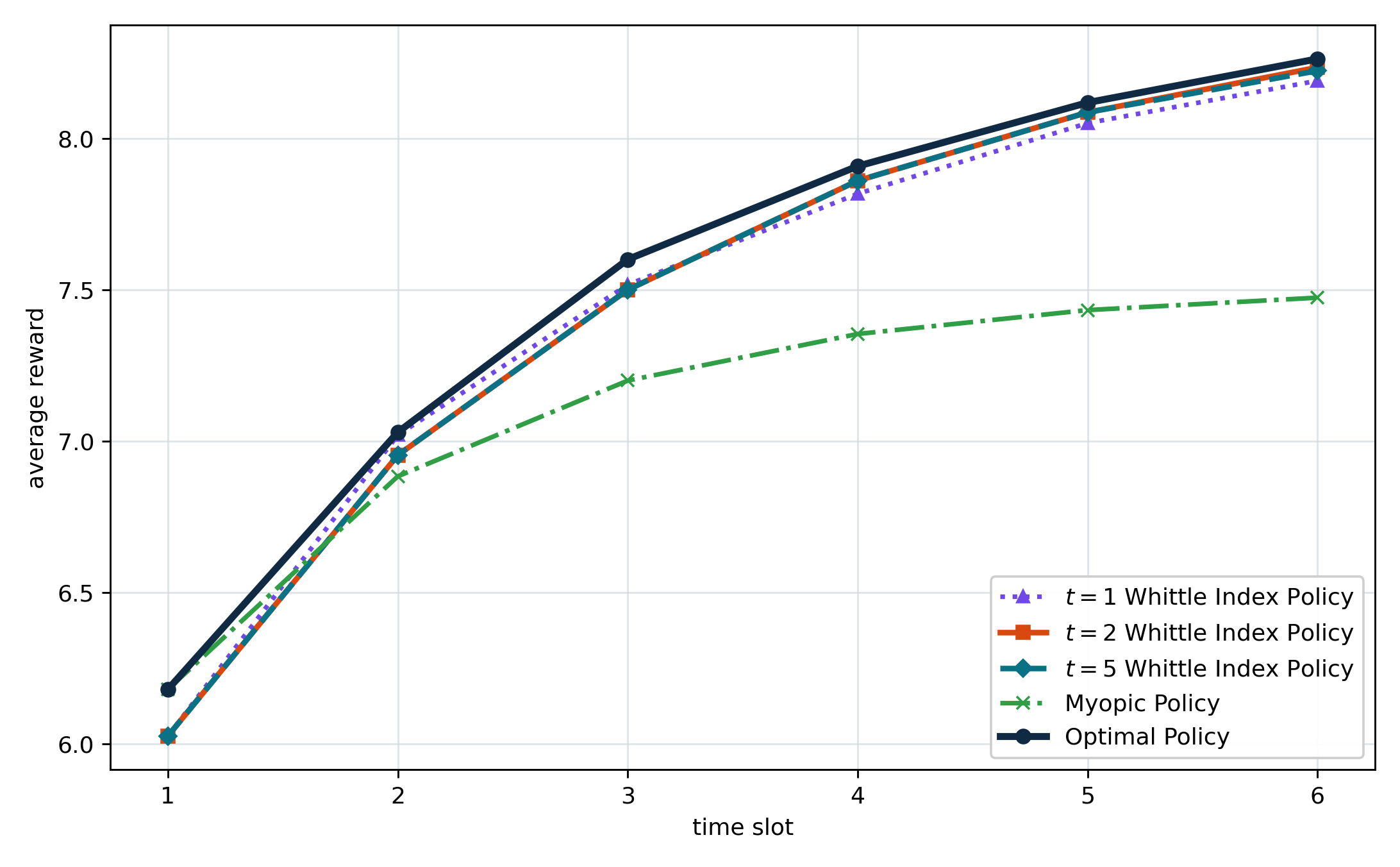}
\caption{Experiment with a large myopic gap.  Both $t=2$ and $t=5$ threshold-index policies outperform myopic by about ten percent and remain close to optimal policy.}
\label{fig:appendix-policy-myopic-gap}
\end{figure}

\begin{figure}[t]
\centering
\includegraphics[width=0.78\linewidth]{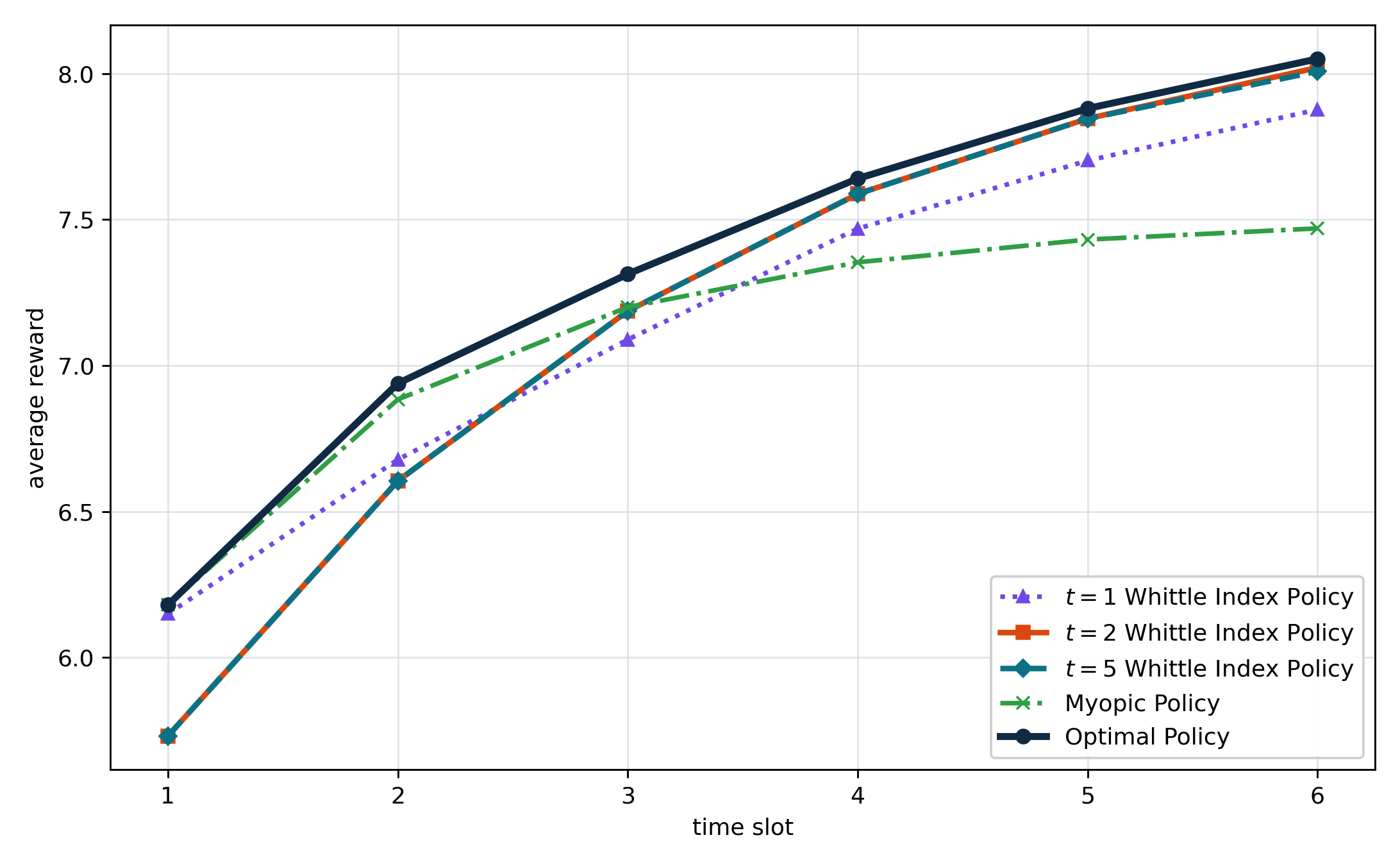}
\caption{Additional one-step separation instance.  The figure provides a same result with Figure 1: deeper threshold-index policies are close to optimal policy and above the one-step threshold at the final horizon.}
\label{fig:appendix-policy-rank2}
\end{figure}

\end{document}